\documentclass[letterpaper, 10 pt, conference]{ieeeconf}  

\IEEEoverridecommandlockouts                             
\usepackage{algorithm}
\usepackage{algorithmic}
\usepackage{amsfonts}
\usepackage{amsmath}
\usepackage{amssymb}
\usepackage{arydshln}
\usepackage{color}
\usepackage{float}
\usepackage{graphicx}
\usepackage{hyperref}
\usepackage{multirow}
\usepackage{textcomp}
\usepackage{threeparttable}
\usepackage{xcolor}

\newtheorem{theorem}{Theorem}
\newtheorem{definition}{Definition}
\newtheorem{lemma}{Lemma}

\newtheorem{assumption}{Assumption}

\title{\LARGE \bf
Distributed Invariant Kalman Filter 
for Object-level Multi-robot Pose SLAM
}

\author{Haoying Li$^{\rm a}$, Qingcheng Zeng$^{ \rm b}$, Haoran Li$^{ \rm a}$, Yanglin Zhang$^{\rm a}$, and Junfeng Wu$^{\rm a}$
\thanks{
$^{\rm a}$:~School of Data Science, The Chinese University of Hong Kong (Shenzhen), Shenzhen, China.
 $^{\rm b}$:~System Hub, Hong Kong University of Science and Technology (Guangzhou), Guangzhou, China.
 Emails: \{haoyingli, haoranli, yanglinzhang\}@link.cuhk.edu.cn (H.Y. Li, H.R. Li, Y. Zhang), qzeng450@connect.hkust-gz.edu.cn(Q, Zeng), junfengwu@cuhk.edu.cn (J. Wu).
} 			
}    

\begin{document}

\maketitle
\thispagestyle{plain}
\pagestyle{plain}

\begin{abstract}
Cooperative localization and target tracking are essential for multi-robot systems to implement high-level tasks. To this end, we propose a distributed invariant Kalman filter based on covariance intersection for effective multi-robot pose estimation. The paper utilizes the object-level measurement models, which have condensed information further reducing the communication burden. Besides, by modeling states on special Lie groups, the better linearity and consistency of the invariant Kalman filter structure can be stressed. We also use a combination of CI and KF to avoid overly confident or conservative estimates in multi-robot systems with intricate and unknown correlations, and some level of robot degradation is acceptable through multi-robot collaboration. The simulation and real data experiment validate the practicability and superiority of the proposed algorithm. The source code is available for download at: \href{https://github.com/LIAS-CUHKSZ/Distributed-object-based-SLAM}{https://github.com/LIAS-CUHKSZ/Distributed-object-based-SLAM}.
\end{abstract}

\section{INTRODUCTION}
The deployment of teams of cooperative autonomous robots has the potential to enable fast information gathering and more efficient coverage and monitoring of vast areas.  In multi-robot systems, high-level features offer advantages over low-level ones, such as broader loop closures between robots. Object-based models also reduce communication by avoiding the exchange of raw sensor data~\cite{choudhary2017distributed} in multi-robot systems.  We tackle the challenge of each robot estimating its trajectory while simultaneously estimating key object poses within a map. These estimates support high-level tasks like semantic scene understanding, sub-map fusion, and multi-robot cooperation for complex tasks like cargo transport.

Multi-robot pose estimation involves complex correlations from shared neighbor information and environmental measurements unrelated to motion noise. 
Designing a fully distributed filter is hindered because, although filtering methods execute quickly, they often struggle with unknown correlations. 
The Kalman filter (KF) and its variants are the most established methods, with numerous efforts made to adapt KF for distributed filtering. The consensus method~\cite{carli2008distributed} converges to a centralized manner but necessitates infinite communication steps. 
The Diffusion KF~\cite{cattivelli2010diffusion} uses convex combinations of estimates during diffusion but lacks covariance updates, rendering the stored covariance meaningless. 
Covariance Intersection (CI)~\cite{chen2002estimation} handles unknown correlations and delivers consistent estimates, though it can be overly conservative in cases of independence. 
Integrating CI with KF offers a more reliable approach. 
Zhu \emph{et al.}~\cite{zhu2020fully} developed a CI-based distributed information filter for localization and tracking, while Chang \emph{et al.}~\cite{chang2021resilient} introduced CIKF, a CI-based distributed Kalman filter that provides consistent multi-robot localization by allowing each robot to estimate the states of the entire team.
\begin{figure}[t]
    \centering
    \includegraphics[width=0.86\linewidth]{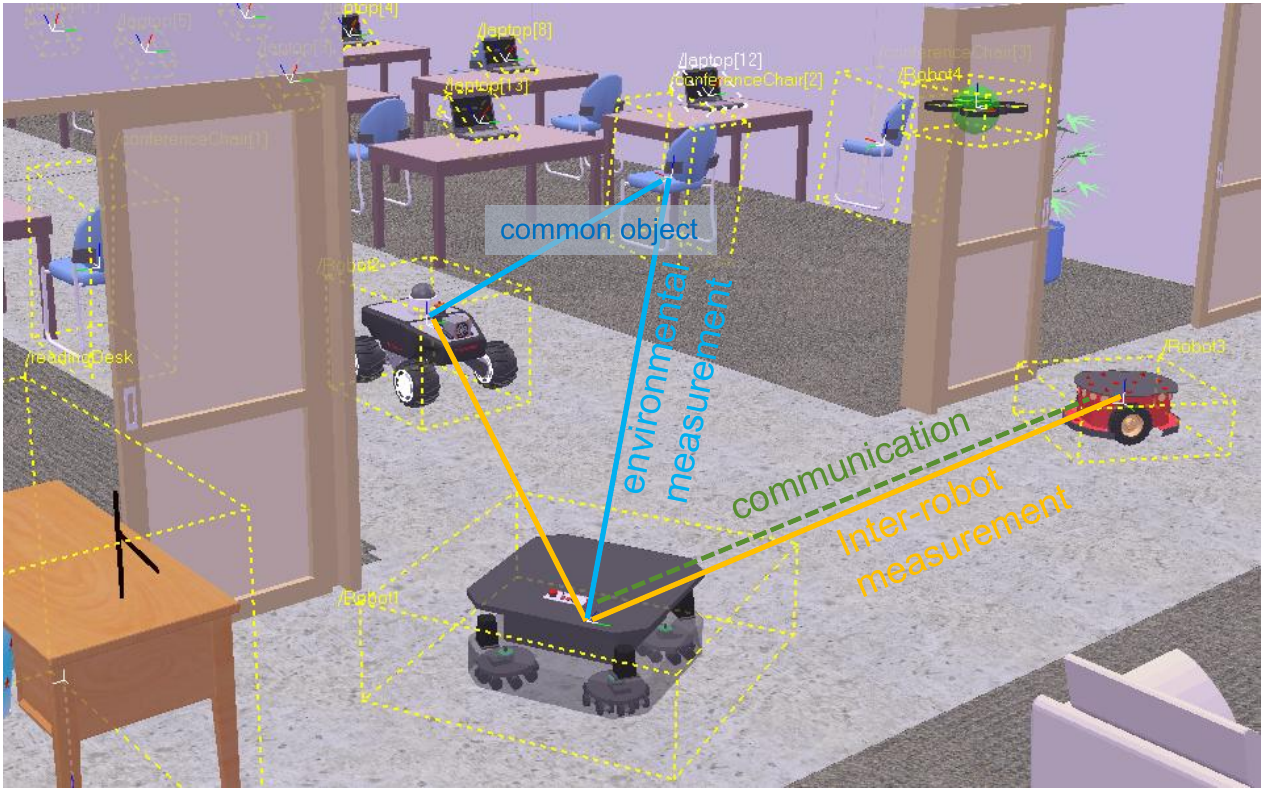}
    \caption{Illustration of distributed object pose SLAM problem.}
    \label{fig:illustration}
\end{figure}

To further enhance pose estimation within multi-robot systems, especially when dealing with movements in the special Euclidean Lie Group, addressing vector space approximation errors is crucial. The invariant extended KF~(InEKF), based on invariant observer theory~\cite{barrau2016invariant}, leverages the log-linear property of invariant errors in group affine systems to improve pose estimation accuracy and consistency, avoiding the inconsistencies in standard EKF. InEKF is also employed in object-based SLAM~\cite{9667208}\cite{jung2022gaussian} and achieved excellent results. InEKF is also extended to distributed filtering. Xu \emph{et al.}~\cite{10232373} introduced a CI-based distributed InEKF for cooperative target tracking on Lie groups, yet neglected relative noises. 
Some related work follows. Cossette et al.~\cite{cossette2024decentralized} proposes a distributed state estimation method using pseudomeasurements and preintegration. In contrast, we use an invariant framework to improve pose estimation accuracy and consistency. Li et al.~\cite{li2021joint} extends SCI to Lie groups for multi-robot localization. Compared to this, we offer a more detailed analysis of observation correlations and leverage object-level observations with the potential to build semantic maps.

In this paper, we present a fully distributed invariant KF based on CI, termed DInCIKF for multi-robot object-level pose SLAM, enabling efficient estimation of both robot states and object states with minimal communication overhead. The effectiveness of our approach is validated through extensive simulations and real-world datasets. The key contributions of this work are summarized as follows:

\begin{enumerate} \item We propose a novel, fully distributed algorithm DInCIKF to estimate robot trajectories and 6-DOF object poses. This approach mitigates the risks of overconfidence and excessive conservatism in the estimates. \item Our state representation is formulated on a specific Lie group, with uncertainties modeled via Lie algebra. This ensures improved linearity, enhancing both the accuracy and consistency of the algorithm. \item The communication framework leverages object-based information, resulting in a highly efficient system that minimizes communication frequency and bandwidth requirements. \end{enumerate}
\section{Preliminaries}
\subsection{Covariance intersection}
CI is a conservative fusion approach that ensures consistent estimates without requiring knowledge of the exact correlations. Given estimates $\hat{x}_1,\ldots, \hat{x}_n$ of $x$ with estimation error covariance $\hat{P}_1,\ldots, \hat{P}_n$, CI fuses them as follows:
\begin{equation}~\label{CI}
\begin{aligned}
 P_{\mathrm{CI}}^{-1}=\sum_{i=1}^{n}\alpha_i \hat{P}_i^{-1}, \quad P_{\mathrm{CI}}^{-1} \hat{x}_{\mathrm{CI}} = \sum_{i=1}^n \alpha_i \hat{P}_i^{-1} \hat{x}_i,
\end{aligned}    
\end{equation}
where $0\leq \alpha_i \leq 1$ and $\sum_{i=1}^n\alpha_i=1$.
The coefficients can be determined by minimizing metrics such as
$\operatorname{tr}(P_{\mathrm{CI}})
$.
\subsection{Matrix Lie group}
Rigid body motion is effectively modeled by matrix Lie groups. A matrix Lie group $G$ is a group as well as a smooth manifold where its elements are matrices. A Lie algebra identified on the identity matrix consists of a vector space $\mathfrak{g}$ with the Lie bracket. The matrix Lie group $SE_n(3)$ is comprised of a rotation matrix $R\in SO(3)$ and $n$ vectors $t_i \in \mathbb{R}^3$, where $SO(3)$ is the set of rotation matrices:$$
SO(3) \triangleq \{R\in \mathbb{R}^3| RR^T = I, \operatorname{det}(R)=1 \}.
$$
An element in Lie algebra of $SO(3)$ is given by:
\begin{small}$$
\omega ^\wedge=\left[\begin{array}{c}
    \omega_1\\ \omega_2 \\
    \omega_3
\end{array}\right] ^\wedge    =\left[\begin{array}{ccc}
0 & -\omega_3 & \omega_2 \\
\omega_3 & 0 & -\omega_1 \\
-\omega_2 & \omega_1 & 0
\end{array}\right].
$$\end{small}

The matrix representations of an element in $SE_n(3)$ and the element in corresponding Lie algebra $\mathfrak{se}_n(3)$ are:
 \begin{small}
$$
\left[\begin{array}{c|lll}
R & t_1&\cdots&t_n\\
\hline
0_{n\times3} &~&I_n
\end{array}\right], \quad
\left[\begin{array}{c|lll}
\omega^\wedge & v_1&\cdots&v_n\\
\hline
0_{n\times 3} &~&0_{n\times n}
\end{array}\right].
$$\end{small}

We abuse the notation $(\cdot)^\wedge$ to denote the mapping from the vector space to the corresponding Lie algebra. The inverse map of $(\cdot)^\wedge$ is denoted as map $(\cdot)^\vee$. The exponential map is defined as $\operatorname{exp}(\xi)\triangleq \operatorname{exp}_m(\xi^\wedge): \mathfrak{g} \rightarrow G$, where $\exp_m(\cdot)$ denotes the matrix exponential. The inverse mapping of exponential is the Lie logarithm, denoted by $\log(\cdot)$.

For $X \in G$ and $\xi \in \mathfrak{g}$, the adjoint map is as follows:
$$
\operatorname{Ad}_X : \mathfrak{g} \rightarrow \mathfrak{g};\quad  \xi ^ \wedge \mapsto \operatorname{Ad}_X (\xi^\wedge)\triangleq X \xi^\wedge  X^{-1}, 
$$
which can yield formulation:
$
  X \operatorname{exp}(\xi ) =\operatorname{exp}(\operatorname{Ad}_X \xi ) X. 
$
Using the Baker-Campbell-Hausdorff (BCH) formula to compose matrix exponentials is complex. For $x^{\wedge}, y^{\wedge} \in \mathfrak{g}$ and $\|y\|$ small, their composed exponential approximates to
\begin{align}
 \exp (x) \exp (y) \approx \exp (x + \operatorname{dexp}_{-x}^{-1} y),~\label{eq:exp_approx_smallbig} 
\end{align}
where $\operatorname{dexp}_x$ is the left Jacobian of $x$.
\section{Problem Formulation}
We consider a multi-robot system consisting of $N$ robots and several object features in the environment. Each robot aims to estimate its trajectory and the poses of key objects. 

We follow the convention of using right subscripts for relative frames and left superscripts for base frames, omitting the global frame. We use a superscript $(\cdot)^{(j)}$ to indicate the source of this estimate is robot $j$. 
\subsection{Kinematics}
Let $p_i$, $R_i$, and $v_i$ denote the position, orientation, and velocity of robot $i$ in the global frame. The robot own state is represented as
\begin{small}
$$
X_{i}\triangleq\left[\begin{array}{c|l}
R_i & p_i\quad v_i\\
\hline
0_{2 \times 3} &~~I_2 
\end{array}\right]\in SE_2(3)
\text{ or }\left[\begin{array}{cc}
R_i & p_i\\
0_{1\times 3}&1
\end{array}\right]\in SE(3).$$ \end{small}

An object feature is represented as a static 6-DOF pose in the environment. Object $f_s$'s pose~(orientation and position) in the global frame is denoted by:
$$
T_{f_s} \triangleq \begin{bmatrix}
    R_{f_s}& p_{f_s}\\0_{1\times3}&1
\end{bmatrix} \in SE(3),
$$
where $R_{f_s} \in SO(3)$ and $p_{f_s} \in \mathbb{R}^3$ are respectively the orientation and the position of the feature in the global frame. The initial state of robot $i$ is given by $Y_i = X_i$. When the robot is tracking $K$ objects, the state expands to
$$
Y_i = \{X_i, T_{f_1}^{(i)},\cdots, T_{f_K}^{(i)}\}.
$$
The relevant state space is isometric to $SE_2(3)$ (or $SE(3)$) $\times (SE(3))^K$. It is assumed that the objects in the environment remain static. For robots, following \cite{barrau2016invariant}, we consider the kinematics of group affine systems in the form of
\begin{equation}
  \dot{{X}}=f_u({X}),  
\end{equation}
where $f_u$ satisfies the so-called group affine property for any $u \in \mathcal{U}$, $X_1, X_2 \in S E_2(3) (\text{or } SE(3))$,
$
f_u\left({X}_1{X}_2\right)={X}_1 f_u\left({X}_2\right)+f_u\left({X}_1\right) {X}_2-{X}_1 f_u(I){X}_2.
$
\subsection{Measurement Models}
We consider object-level measurements, which represent relative poses of object features in the current robot frame. The first type of the measurement model describes environmental measurements of object features modeled as:
\begin{equation}~\label{eq:z_if}
^iT_{f_s,m}\triangleq h_{i f}(X_i, T_{f_s})=  T_i^{-1} \exp (n_{i f})  T_{f_s} ,
\end{equation}
where the measurement noise is modeled as white Gaussian noise $n_{if}\sim \mathcal{N}(0, R_{if}) $ and we use $T_i \in SE(3)$ denote the pose of robot $i$. 

The second type is the inter-robot measurement between neighboring robots $i$ and $j$, which is modeled as:
\begin{equation}~\label{eq:z_ij}
^iT_{j,m}\triangleq h_{i j}(X_i, X_j)= T_i^{-1} \exp(n_{ij}) T_j , \quad j \in \mathcal{N}_i,
\end{equation}
where the relative measurement noise $n_{i j} \sim \mathcal{N}(0, R_{i j})$.
\subsection{Problem }
The communication network of a multi-robot system is defined as \(\mathcal{G}=(\mathcal{V}, \mathcal{E})\), where \(\mathcal{V}\) includes $N$ robots and an additional node \(o\) representing absolute pose information in the environment. The communication links are denoted by \(\mathcal{E}\), with \((j,i) \in \mathcal{E}\) indicating that robot \(i\) can receive information from robot \(j\), and \((o,i) \in \mathcal{E}_{oi}\) signifying that node \(i\) has access to absolute pose information. The set of neighboring robots for robot \(i\) is denoted as \(\mathcal{N}_i\). 
We assume that \(\mathcal{G}\) is time-invariant and make the following assumptions.
\begin{assumption}~\label{a:measure_can_send}
If robot $j$ is capable of measuring the relative pose of robot $i$, it can send information to robot $i$, that is, $(j, i)\in\mathcal E$.
\label{a1}
\end{assumption}
\begin{assumption}\label{a:spanning_tree}
The digraph ${\mathcal{G}}$ contains a directed spanning tree rooted at node $o$. 
\label{a2}
\end{assumption}

We are interested in \textit{designing a distributed pose estimator to compute $Y_i$ for each robot $i$} using the available measurements and leveraging occasional communication with other robots. The problem of interest is shown in Fig.\ref{fig:illustration}.
\section{Methodology}
The distributed estimator procedure consists of three steps. We use $\bar{(\cdot)}$ to denote the estimate after local invariant error preintegration, $\breve{(\cdot)}$ for the estimate following the invariant KF update, and $\hat{(\cdot)}$ for the estimate after communication with neighbors. An overview of the system is provided in Fig.\ref{fig:system}.
\begin{figure}
    \centering
    \includegraphics[width=0.9\linewidth]{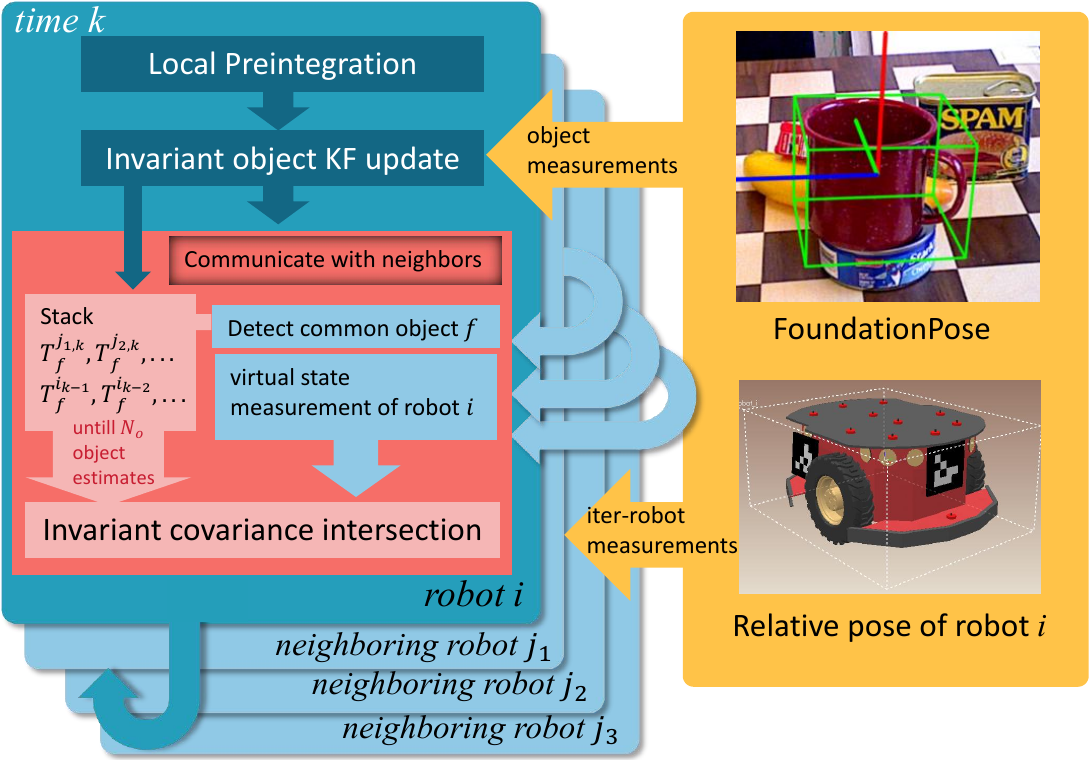}
    \caption{System overview.}
    \label{fig:system}
\end{figure}
\subsection{Local Invariant Error preintegration}
Define the logarithmic right invariant error as
$\xi_i =\operatorname{log}(X_i \bar{X}_i^{-1})$. Similarly, the logarithmic right invariant error from $\bar{T}_{f_s}^{(i)}$ to $T_{f_s}$ is $\delta_{f_s}^{(i)} \triangleq \operatorname{log}(T_{f_s} (\bar{T}_{f_s}^{(i)})^{-1})$. The logarithmic right invariant error for the whole state $Y_i$ is then as follows 
$$
\zeta_i = [\xi_i~ \delta_{f_1}^{(i)} ~\cdots ~ \delta_{f_K}^{(i)} ] \in \mathfrak{s e}_n(3) \times (\mathfrak{s e}(3))^K, n=1 \text{ or }2.
$$
We use $d$ to denote the dimension of $\xi_i$, with $d=6$ for $n=1$ and $d=9$ for $n=2$ in the following context.
Many common process models in robotics, which conform to affine group properties, can utilize preintegration. Here are a few examples that illustrate the preintegration of invariant error and Lie group state.
\subsubsection{Example 1} Direct pose increment on $SE(3)$. This first-order process model employs the relative motion increment (RMI) for preintegration~\cite{sola2018micro}~\cite{cossette2024decentralized}, with the RMI 
can be derived from constant velocity assumptions~\cite{9667208} or odometry between frames. 
The kinematics of robot $i$ given by:
\begin{align*}
    \dot{p}_i &= \omega_i^\wedge + v_i,\quad\dot{R}_i= \omega_i^\wedge R_i,
\end{align*}
where $v_i$ and $\omega_i$ are the translational and rotational velocities. Let $u_i\triangleq[v_i \quad \omega_i]$ be the generalized velocity, it follows
\begin{equation}~\label{eq:process1}
    \dot{{T}_i}=u_i^{\wedge} {T_i}.
\end{equation}
If we assume quantities remain constant between discrete times, then the discrete state preintegration is as follows:
$$
\bar{T}_i(k)  = \hat{T}_i(k-1) \Delta T_{i,m}(k)= \hat{T}_i(k-1)\operatorname{exp}( u_i(k)\Delta t),
$$
where the $\Delta t$ is the constant sampling interval. The control input $\Delta T_{i,m}(k)=\operatorname{exp}(n_{i,u}) \Delta T_i(k)$ with the process noise $n_{i,u} \sim \mathcal{N}(0, S_{i,u})$. We can obtain
\begin{align*}
    \bar{T}_i(k) &= \exp{(-\hat{\xi}_i(k-1))}T_i(k-1) \exp{(n_{i,u}) } \Delta T_i(k)\\
    & = \exp{(-\hat{\xi}_i(k-1))}\exp{(\operatorname{Ad}_{ T_i({k-1})}(n_{i,u}))}\Delta T_i(k).
\end{align*}
Following~\cite{li2021joint}, we use the first-order estimate and then the propagation of the estimation error covariance $P_{ii}$ corresponding to $\xi_i$ is as follows:
\begin{align*}
    \bar{P}_{ii}(k)&=\hat{P}_{ii}(k-1) + \operatorname{Ad}_{\hat T_i(k-1)} S_{i,u} \operatorname{Ad}_{\hat T_i(k-1)}^\top\\
    &=F_i\hat{P}_{ii}(k-1)F_i^\top + S_i(k),
\end{align*}
where $F_i \triangleq I_6$ and $S_i(k) \triangleq \operatorname{Ad}_{\hat T_i(k-1)} S_{i,u}\operatorname{Ad}_{\hat T_i(k-1)}^\top$.
\subsubsection{Example 2} IMU preintegration on $SE_2(3)$. This is a second-order process model, which is the most well-known usage of preintegration. The 3D kinematics for robot $i$ is:
\begin{align}~\label{eq:IMU_process}
\dot{R}_i= R_i \omega_i^\wedge,\quad \dot{p}_i &= v_i, \quad \dot{v}_i = a_i,\\
\omega_{i,m}=\omega_i+b_g+n_g,\quad a_{i,m}& = R_i^T (a_i-g)+b_a+n_a,~\notag
\end{align}
where $w_i$ and $a_i$ are angular velocity and acceleration, measured by the IMU. The IMU measurements $w_{i,m}$ and $a_{i,m}$ include zero-mean white Gaussian noises $n_g$, $n_a$, and $g$ denotes the gravity. The gyroscope and accelerometer bias $b_g$, $b_a$ are driven by white Gaussian noises $n_{b_g}$ and $n_{b_a}$, that is, $ \dot{b}_g=n_{b_g}, \dot{b}_a = n_{b_a}
 $. We can follow Appendix A in~\cite{hartley2020contact} to discretely propagate the group state without process noises to obtain $\bar{X}_i(k)$ from $\hat{X}_i(k-1)$. 

The discretization of  dynamics of $\xi_i$ is given by~\cite{li2022closed}:
\begin{align*}~\label{eq:xi_prop_discrete}
  \bar{\xi}_i(k)=F_i\hat{\xi}_i(k-1)+n_{i}(k),
\end{align*}
with $F_i \triangleq $ \begin{footnotesize}$\begin{bmatrix}
    I & 0 & 0\\
    \frac{1}{2}g^\wedge \Delta t^2&I& I \Delta t\\ g^\wedge \Delta t &0&I
\end{bmatrix}$\end{footnotesize}, $n_{i}(k) \sim \mathcal{N}(0_{9 \times 1}, S_i(k))$ can be obtained by \cite{li2022closed}. We thus use the following equation to approximate the evolution of the estimation error covariance:
\begin{equation}~\label{pred_P}
    \bar{P}_{ii}(k)=F_i \hat{P}_{ii}(k-1) F_i^\top+S_i(k) .
\end{equation}
Assuming constant sampling intervals, the $F_i$ of the robots are identical and time-invariant.
When estimating the pose of an object simultaneously, the $F_i$ needs to be augmented. The preintegration after augmentation is described as follows. Using $P_i$ denote the estimation error covariance of $\zeta_i$ including estimating $K$ objects. Then $P_i$ propagates as
\begin{equation}~\label{eq:whole_propagate}
\bar{ P}_i(k) = A_i \hat{P}_i (k-1) A_i^\top+Q_i(k),    
\end{equation}
where $A_i = \operatorname{diag}(F_i, I_{6K})$, $Q_i(k) = \operatorname{diag}(S_i(k), 0_{6K\times6K})$.
\subsection{Invariant Kalman Filter Update}
Next, we use the invariant Kalman filter update step to fuse observations from the environment that are not coupled with process noise.
\subsubsection{New Object Initialization}
After confirming the reliability and importance of the object, the robot will initialize it based on the measurement from the first observation. Note that the measurement used for initializing the object's state will not be utilized in the update process.

At time step $k$, if a new object feature $f_{K+1}$ is recognized, the state of the feature will be initialized as follows:
\begin{align}~\label{eq:Poseprop}
  \breve{T}_{f_{K+1}}(k)& =  \bar{T}_i(k)   ^i T_{f_{K+1},m}(k),
\end{align}
Consequently, the estimation error covariance will evolve as follows:
\begin{align}
\breve{P}_{f_{K+1}}^{(i)}(k)& =G \bar{P}_i(k) G^\top + R_{if},~\label{eq:initial}    \\
\breve{P}_{i:f_K, f_{K+1}}&= \bar{P}(k) G_f(k)^\top,~\notag
\end{align}
 where $G\triangleq[I_{6}~ 0_{6\times 3}]$ for $X_i\in SE_2(3)$, $G\triangleq I_6$ for $X_i \in SE(3)$ and $G_f(k)\triangleq \left[\begin{array}{cc}
      \multicolumn{2}{c}{I_{d+6K}}  \\
     G & 0_{6\times 6K}
 \end{array}\right]$.
\subsubsection{Local Update}
For previously initialized objects, their measurements can be used for updates. Define the residual as:
\begin{equation}
r_{i f_s}= \log( ^{i}T_{f_s,m}  ( 
\bar{T}_i^{-1}\bar{T}_{f_s}^{(i)})^{-1} )
\end{equation}
For residual at time $k$, we have
$$
r_{i f_s}(k) = H_{i f_s}(k) \bar{\zeta}_i(k) + v_i(k),
$$
where the Jacobian matrix of $r_{if_s}(k)$ with respect to $\xi$ is
$$
H_{i f_s}(k) = [J_i(k)~ 0~\cdots~J_{i f_s}(k)~\cdots~0],
$$
where $J_{i f_s}(k) = -\operatorname{Ad}_{\bar{T}_i(k) }$, $J_i(k) =\operatorname{Ad}_{\bar{T}_i(k) }J $ and $v_i(k) \sim \mathcal{N}(0, R_{if}'(k))$ with $R_{if}'(k)= \operatorname{Ad}_{\bar{T}_i(k) } R_{if} \operatorname{Ad}_{\bar{T}_i(k) }^{\top} $. The calculation details see Appendix.\ref{app:Jacob}.

The stacked forms of the residual, Jacobian matrix and the noise covariance are denoted by $r_i(k)$, $H_i(k)$ and $R_i(k)$. For robots that have access to their own absolute pose with measurement model $f_o(X_i)$ and residual $r_o(X_i)$, the observation matrix is $H_{io}=[J_{io}~ 0]$, where $J_{io}$ is the Jacobian matrix of $r_o$ with respect to $\xi_i$. Absolute measurements are also taken into account. Subsequently, we then use the standard KF update procedure:
\begin{align}~\label{eq:updateobjp}
K_i(k) &= \bar{P}_i(k) H_i(k)^\top( H_i(k)  \bar{P}_i(k) H_i(k)^\top +R_i(k) ) ^{-1},~\notag\\
\breve{P}_i(k) &= (I-K_i(k) H_i(k) )\bar{P}_i(k),\\
\breve{\zeta}_i(k) & = K_i(k)r_i(k) + \bar{\zeta}_i(k)= \Delta{\zeta}_i(k)+ \bar{\zeta}_i(k).~\notag
\end{align}
With $\Delta\zeta_i(k)=[\Delta\xi_i, \Delta\delta_{f_i}^{(i)}, \cdots, \Delta\delta_{f_K}^{(i)}]$, $K$ is the number of estimated objects, mapping back to Lie group, we get 
\begin{align}~\label{eq:mapback}
  \breve{X}_i(k)&= \exp{(\Delta\xi_i(k))}\bar{X}_i,  \\
  \breve{T}_{if_s}^{(i)}(k)&= \exp{(\Delta\delta_{f_s}^{(i)}(k))}\bar{T}_{if_s}(k), s=1,\ldots,K.~\notag
\end{align}

\begin{algorithm}[t] 
\caption{DInCIKF for Object-based Pose SLAM} 
    \label{algorithm} 
    \begin{algorithmic}
    \REQUIRE
    (local) $\hat{Y}_i(k-1), \hat{P}_i(k-1), u_i(k),     ^iT_{f,m}(k); $\\
    (neighbors broadcast) $\breve{T}_{\alpha}^{(j)}(k), \breve{P}^{(j)}_{\alpha}(k), j\in \mathcal{N}_i$.
    \ENSURE $\hat{Y}_i(k), \hat{P}_i(k)$.
    \STATE \textbf{\underline{Step 1: Preintegration.}} 
    \STATE Using the control input $u_i(k)$ to preintegrate $\bar{Y}_i(k)$ and $\bar{P}_i(k)$ using \eqref{eq:IMU_process} or \eqref{eq:process1} and \eqref{eq:whole_propagate}.
    \STATE \textbf{\underline{Step 2: Invariant object update.}}
    \STATE (1) Initialize the first-seen objects's state $T_{f_{K+1}}^{(i)}$ and $P_{f_{K+1}}^{(i)}$ following~\eqref{eq:Poseprop} and~\eqref{eq:initial};
    \STATE (2) Update the $\breve{Y}_i(k)$ and $\breve{P}_i(k)$ following~\eqref{eq:mapback}, \eqref{eq:updateobjp}.
    \STATE \textbf{\underline{Step 3: CI update.}}
    \STATE (1) Compute $\tilde{T}_{\alpha}(k) $ and $\tilde{P}_{\alpha}(k)$ following~\eqref{eq:fuseCIupda}. 
    \STATE (2) Compute $\tilde{r}_i(k)$ and $\tilde{R}_i(k)$ following~\eqref{eq: res_neigh};
    \STATE (3) Update the $\hat{Y}_i$ and $\hat{P}_i$ following~\eqref{eq:mapback}and~\eqref{eq:CIup}.
\end{algorithmic} 
\end{algorithm}
\subsection{Communicating with Neighbors}
\subsubsection{Neighboring Pose Fusion} 
During the communication step, robot $i$ receives state estimates from its neighbors, along with pose estimates of common objects. 
To incorporate these estimates as direct measurements, inspired by~\cite{xu2023distributed}, robot $i$ uses a fast CI strategy on the Lie group, accounting for correlations in its neighbors' measurements.

The pose fusion aims to fuse the set of pose estimations from robot $j$ denoted by $T_{\alpha}^{(j)}, \alpha = i \cup \{f_c\}$ with corresponding covariance $P_{\alpha}^{(j)}$, where $c\in I_i$, and $I_i$ is the index set of the common estimated objects of robot $i$ with its neighbors. 

 For the state of robot $i$, neighbor $j$ can locally estimate $i$'s pose by $$
\breve T_i^{(j)}(k) = \breve{T}_j(k) ^{j} T_{i,m}(k),
$$ 
with covariance $\breve{P}_{i}^{(j)}(k) = G \breve{P}_{jj}(k) G^\top +R_{ij}$. 

For commonly observed objects, the poses awaiting fusion include robot $i$'s historical estimates of the object and its neighbors' estimates of the same object. Fusion of $\breve{T}_{f_c}^{(j)}$ covariance $\breve{P}_{f_c}^{(j)}$ with historical $i$'s estimation will only begin once $ N_o$ pose estimates for the object have been collected. 

The information received from the neighbors is fused to become a virtual direct observation that forms part of the $Y_i$:
\begin{align}~\label{eq:fuseCIupda}
&\tilde{P}_\alpha(k)^{-1}=\sum_{j \in \mathcal{N}_i} \beta_{j,k} \breve{P}_\alpha^{(j)}(k) ^{-1} ,~\notag\\
&\tilde{\delta}_\alpha(k)=\tilde{P}_\alpha(k)\sum_{j \in \mathcal{N}_i} \beta_{j,k}(\breve{P}_\alpha^{(j)}(k)^{-1} \log (\breve{T}_\alpha(k) {\breve{T}_{\alpha}}^{(j)}(k)^{-1}  ))~\notag \\
& \tilde{T}_\alpha(k) = \exp{(\tilde{\delta}_\alpha(k))}\breve{T}_\alpha(k),
\end{align}
where $\beta_{j,k}$ is chosen by fast CI rule~\cite{niehsen2002information}. Note that past estimates from robot \(i\) may sometimes be included in the object's state estimation in a similar way.
\begin{figure*}[ht] 
    \centering 
    \begin{minipage}[t]{0.09\textwidth}
        \centering
        \includegraphics[width=\textwidth]{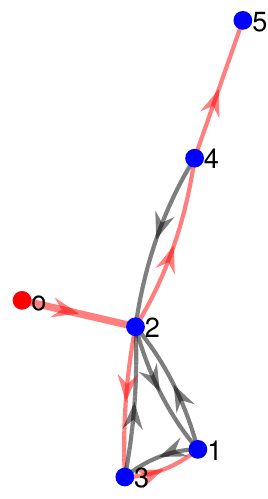}
        \text{(a)} 
        \label{fig:com_graph}
    \end{minipage}
    \hfill
    \begin{minipage}[t]{0.45\textwidth}
        \centering
        \includegraphics[width=\textwidth]{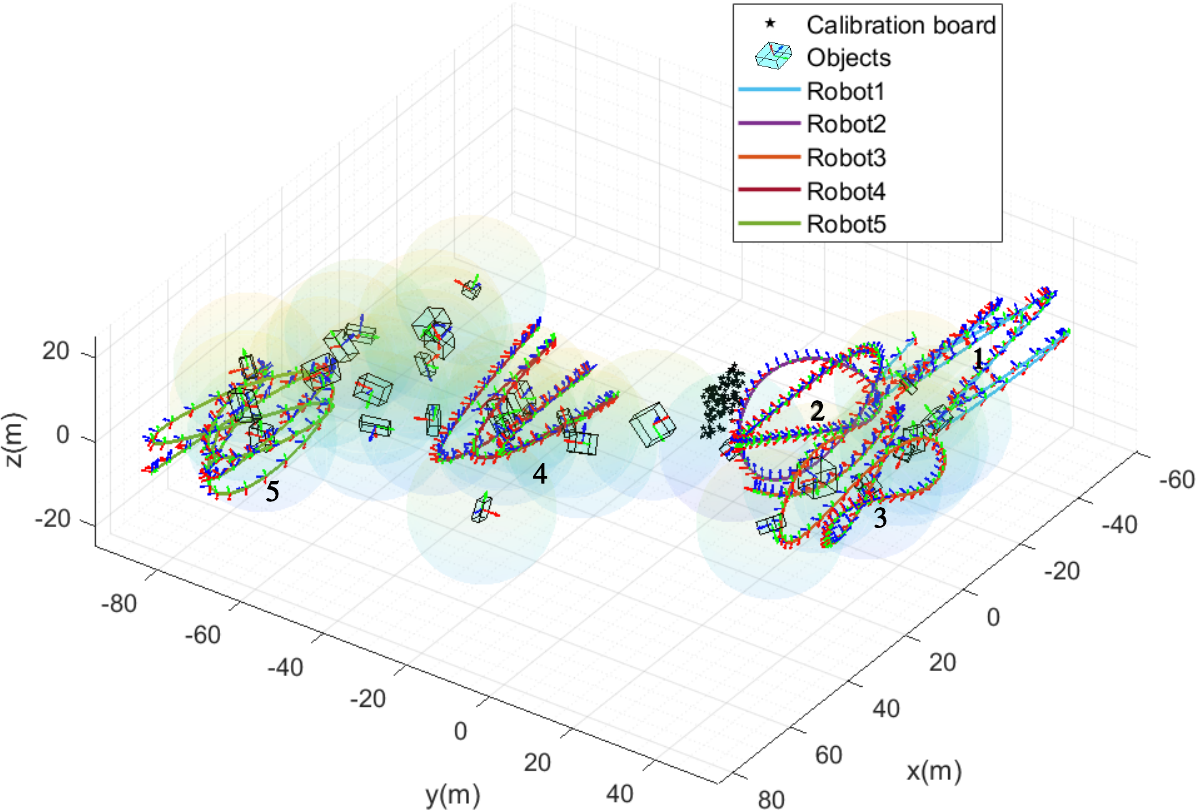}
        \text{(b)}
        \label{fig:traj}
    \end{minipage}
    \begin{minipage}[t]{0.39\textwidth}
        \centering
        \includegraphics[width=\textwidth]{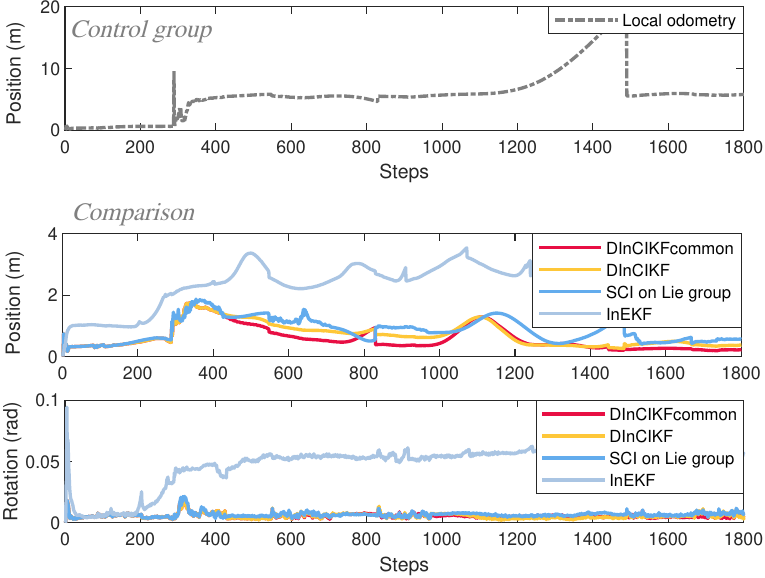}
        \text{(c)}
        \label{fig:rmse_curve}
    \end{minipage}
    \caption{\textbf{Simulation Experiment 1:} (a) Communication graph showing robot 2 with access to absolute information (points on a calibration board). The red line indicates a spanning tree rooted at the absolute source node $o$. (b) Simulation scenario with five robot trajectories and 38 object-level features. The sphere represents the 15-meter visible range of objects. Robots 1 and 5 lack object-level observations for extended periods, simulating a degenerate environment. (c) RMSE curves for average robot position and rotation estimation. The control group uses local odometry without communication, leading to inaccuracies and divergence from a prolonged lack of observation. Among the methods compared, ours achieves the highest localization accuracy.}
    \label{fig:sim1}
\end{figure*}
\begin{figure*}[ht] 
\vspace{-3mm} 
    \centering
    \begin{minipage}[t]{0.09\textwidth}
        \centering
        \includegraphics[width=\textwidth]{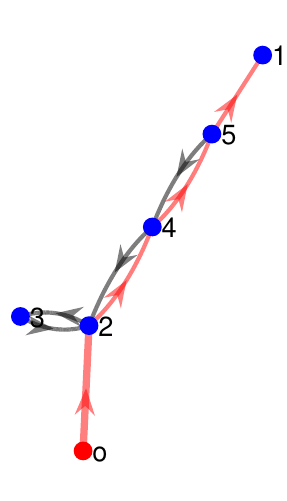}
        \text{(a)} 
        \label{fig:comgraph}
    \end{minipage}
    \hfill
    \begin{minipage}[t]{0.415\textwidth}
        \centering
        \includegraphics[width=\textwidth]{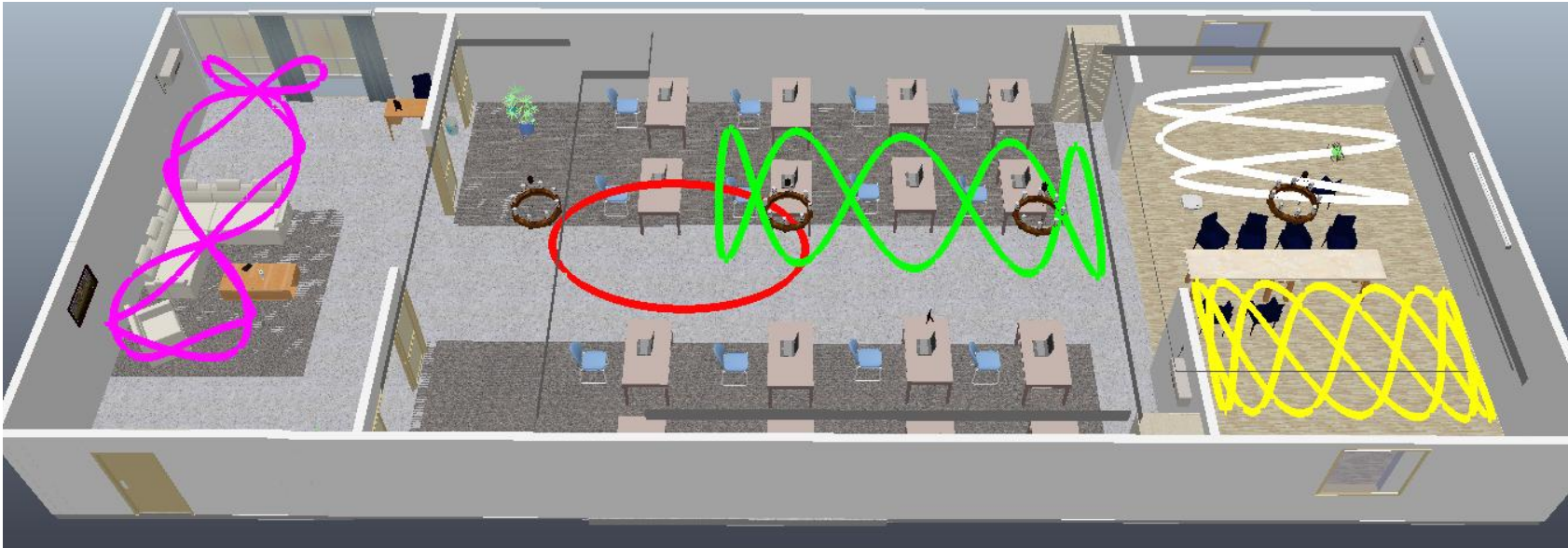}
        \text{(b)}
        \label{fig:sen2}
    \end{minipage}
    \hfill
    \begin{minipage}[t]{0.42\textwidth}
        \centering
        \includegraphics[width=\textwidth]{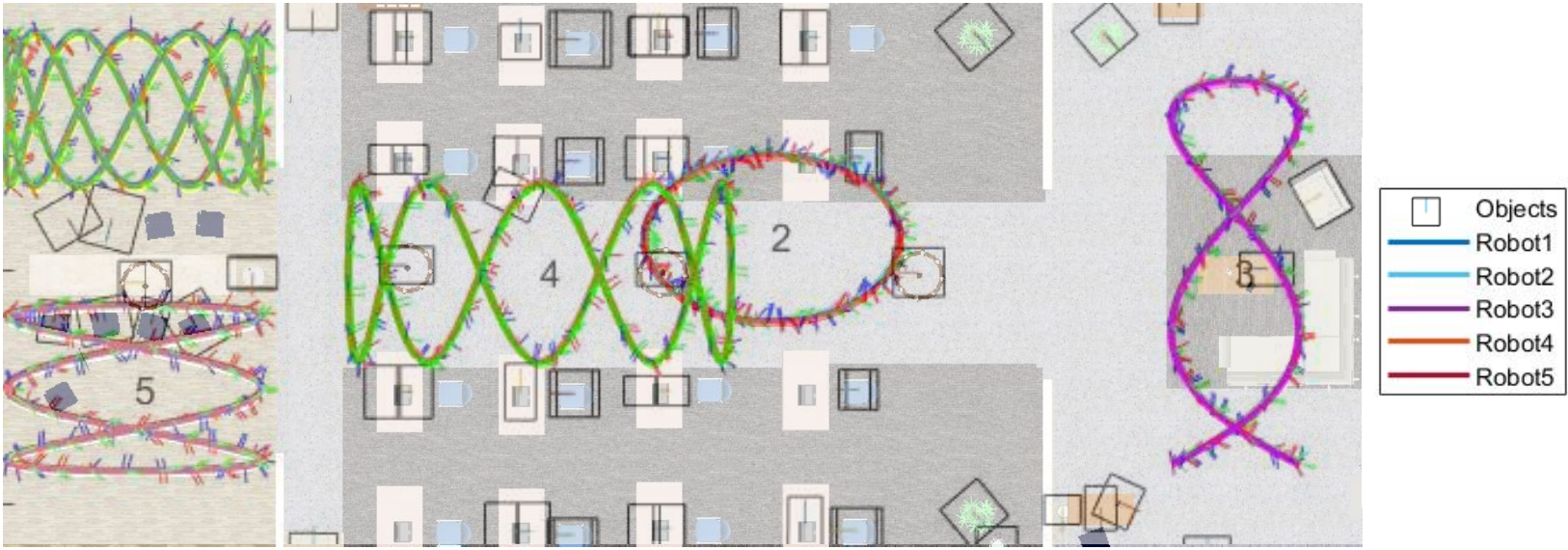}
        \text{(c)}
        \label{fig:path}
    \end{minipage}
    \caption{\textbf{Simulation Experiment 2: }(a) Communication graph of scenario 2. (b) 3D Simulation scene in CoppeliaSim. Five robots are working collaboratively in the office. (c) Illustration of estimated robot trajectories and object pose estimation results.}
    \label{fig:sim2}
    \vspace{-5mm}  
\end{figure*}

Next, CI on the Lie group is applied to fuse two elements: the neighbors' estimates of part of robot $i$'s state and robot $i$'s current estimate. Given the fused partial observations from neighbors, assuming one common object $f_s$ is fused for now, the method extends easily to multiple objects. The residual in CI step and the virtual noise covariance are defined as:
\begin{align}~\label{eq: res_neigh}
\tilde{r}_{i}(k) &= \begin{bmatrix}
    \log (\breve{T}_{i}(k)\tilde{T}_{i}(k)^{-1}  )\\
    \log (\breve{T}_{f_s}^{(i)}(k)\tilde{T}_{f_s}(k)^{-1}  )
\end{bmatrix},~\notag \\
\tilde{R}_i(k) &= \operatorname{diag}(\tilde{P}_i(k),\tilde{P}_{f_s}(k)). 
\end{align}
Then the fusing strategy is 
\begin{align}~\label{eq:CIup}
&\hat{P}_i(k)^{-1}=\gamma_{i,k}\breve{P}_i(k)^{-1}+(1-\gamma_{i,k})J_{\alpha,k}^\top\tilde{R}_i(k)^{-1} J_{\alpha,k},~\notag\\
&\hat{\zeta}_i(k)=(1-\gamma_{j,k})\hat{P}_i(k)J_{\alpha,k}^\top \tilde{r}_i(k)  , 
\end{align}
where \begin{small}$J_{\alpha,k} =\left[\begin{array}{c|ccc}
    I_6 &\multicolumn{3}{c}{0_{6\times[(d-6)+6K]}}  \\
    \hline
    0_6 &0~\cdots~0& I_6&0~\cdots~0 
\end{array} \right]\in \mathbb{R}^{12\times(d+6K)}$\end{small} and $\gamma_{j,k}>0$ is determined by $\min \operatorname{tr}(\hat{P}_i(k) )$. And $\breve{\zeta}_i(k)$ is omitted for zero means.
The last step is to map back to the Lie group state, the same as~\eqref{eq:mapback}.

The whole procedure of the proposed algorithm~DInCIKF for object-based pose SLAM is summarized in Algorithm 1.
\section{Theoretical Analysis}
We apply the result of DInCIKF without estimating object poses in~\cite{li2024covarianceintersectionbasedinvariantkalman} to analyze the covariance boundedness of $\hat{P}_i(k)$.
We can show that the DInCIKF applied in estimating robots' and objects' states is stable regarding the expected squared estimation error over time. 
\begin{theorem}[Stability of DInCIKF]
Under Assumption\ref{a:spanning_tree}, DInCIKF is stable in the sense that $\mathbf{E}[\hat{\xi}(k) \hat{\xi}(k)^{\top}]$ is bounded across time $k$ for all robots.
\end{theorem}

The logic of the proof is to establish the stability of DInCIKF by demonstrating the stability of an auxiliary upper bound system~(AUBS).
For clarity, in the following context, we analyze the case where each robot estimates one object $f$, that is, $Y_i =[X_i~T_f]$ and the communication procedure does not fuse the common object estimates as a worst case. 
This analysis can be readily extended to the general case.
We use $Q_i$ and $R_i$ to denote the limited upper bound of the noise covariances. 

\textit{Proof of Theorem 1.} After the KF and CI step, the \textit{a posteriori} estimation error covariance is given by
\begin{small}
\begin{align*}
 &\hat{P}_i(k)^{-1}\\
 &~=\gamma_{i,k} \breve{P}_i(k)^{-1}+\left(1-\gamma_{i, k}\right) J_1^{\top}\tilde{P}_i(k)^{-1} J_1 \\
&~= \gamma_{i, k}\left(\bar{P}_i(k)^{-1}+H_i(k)^{\top} R_i^{-1} H_i(k)\right) \\
&~~+\left(1-\gamma_{i, k}\right)  \\
&~J_1^{\top}\sum_{j \in \mathcal{N}_i} \beta_{j, k}\left(J_1\left(A \hat{P}_{j}(k-1) A^{\top}+Q_j\right) J_1^{\top}+R_{i j}\right)^{-1} J_1,
\end{align*}
where $J_1\triangleq[G~0_{6\times6}]$ and robots share the same $A\triangleq \operatorname{diag}(F~0_{6 \times 6})$ for all estimating one object.
\end{small}

Under Assumption\ref{a2}, we can construct a directed spanning tree \(\mathcal{T} = (\mathcal{V}, \mathcal{E}_T)\) in \(\mathcal{G}\) rooted at node \(o\). The set \(\mathcal{V}\) is partitioned into \(\mathcal{V} = \mathcal{V}_1 \cup \mathcal{V}_2 \cup \cdots\), where \(\mathcal{V}_1\) consists of the root’s child nodes, \(\mathcal{V}_2\) includes the children of \(\mathcal{V}_1\), and so on. Denote \(\mathcal{N}_{io}\) as node \(i\)’s neighbors and its absolute information source, if accessible. Then the AUBS is constructed as follows:
\begin{equation}~\label{AUBS}
\begin{aligned}
& \bar{\Pi}_i(k)=A\hat{\Pi}_i(k-1)A^\top +Q_i, \\
& \hat{\Pi}_i(k)^{-1}=\alpha_{i}\bar{\Pi}_i(k)^{-1}+\check{H}_i(k)^{T}\check{R}_i^{-1} \check{H}_i(k),
\end{aligned}
\end{equation}
initialized with 
$\bar{\Pi}_{i}(0) \geq \bar{P}_{i}(0)$. In \eqref{AUBS}, if $i \in \mathcal{V}_1$, $\check{H}_{i}(k)\triangleq \begin{bmatrix}
{H}_{io}(k)\\
H_i(k)    
\end{bmatrix}$ and $\check{R}_i\triangleq \operatorname{diag}(R_{io}~ R_{if})$. By recursive definition, provided a node $j$'s $\check{H}_j(k-1)$ and $\check{R}_j(k-1)$ has been defined, for any of its child node $i$, the observation matrix includes neighbor-related $H_{i,n}$ part and its local part:
\begin{align*}
\check{H}_i(k) &\triangleq \begin{bmatrix}
    H_{i,n}(k)\\
    H_{i}(k)
\end{bmatrix} = \begin{bmatrix}
    \check{H}_{j}(k-1) A^{-1} J_1^\top J_1\\
    H_i(k)
\end{bmatrix}. 
\end{align*}
By~\cite {kailath2000linear}, the boundness of $\hat{\Pi}_i$ is decided by the observability of the AUBS. We recall the definition of the observability matrix in the following.
\begin{definition}[Observability~\cite{antsaklis2007linear}]
$(A, H)$ is observable if and only if for any $k>0$, there exists $n_k\geq 1$, such that the observability matrix, $O(A,H)=[H(k)^{\top},(H(k+1)A)^{\top},(H(k+2)A^2)^{\top},...,(H(k+n_k)A^{n_k})^{\top})]^{\top}$ has full column rank. 
\end{definition}

The next lemma illustrates the link between the observability of the robot's state and the joint observability of both the robot and the object. 
\begin{lemma}
If robot $j_1\in \mathcal{V}_1$ and $O({F}, J_{j_1,n})$ has full column rank, then $O({A}, \check{{H}}_i)$ has full column rank for all robots.
\label{lem:relation_fullrank}
\end{lemma}
\begin{proof}
See Appendix.\ref{App:pf_fullrank}.    
\end{proof}

By proving the observability of AUBS, we can conclude that its covariance converges and is bounded. We then introduce the following lemma to link the DInCIKF and the AUBS. 
\begin{lemma}
$
\hat{\Pi}_i(k) \geq \hat{P}_i(k)$ and $    \bar{\Pi}_i(k) \geq \bar{P}_i(k), ~ \forall k \geq 0.
$
~\label{Lem:upper}
\end{lemma}
\begin{proof}
See Appendix.\ref{App:pf_upper}.
\end{proof}
By Lemma~\ref{Lem:upper} and the convergence of the AUBS, we can finally conclude the proof.  $\blacksquare$
\section{Experiment}
\subsection{Simulation}
This section introduces the simulation experiments. 
We compare our proposed DInCIKF, both with common features among neighbors (DInCIKFc) and without (DInCIKF), against InEKF~\cite{li2022closed} and SCI on Lie groups (SCI Lie)~\cite{li2021joint}. We also include local odometry as a control group. Following~\cite{chang2021resilient}, we evaluate the estimation results at each time step $k$ using the average root-mean-squared-error (RMSE) of rotation and position across $N$ robots as follows:
\begin{small}
\begin{align*}
\operatorname{RMSE}_{rot}(k)&=\sqrt{\frac{ \sum_{i=1}^N\| \log(\hat R_i(k)^T R_i(k))   \|^2 }{N}},\\
\operatorname{RMSE}_{pos}(k)&=\sqrt{\frac{ \sum_{i=1}^N\| \hat P_{ii}(k)- P_{ii}(k)  \|^2 }{N}}.
\end{align*}
\end{small}The average RMSE of rotation and poison among all time steps are denoted as $\overline{\operatorname{RMSE}}_{rot}$ and $\overline{\operatorname{RMSE}}_{pos}$. 
\begin{table*}[tp]
    \centering
    \caption{RMSE of Simulation 1 and Simulation 2}\vspace{-3mm}  
    \begin{tabular}{c|ccccc|ccccc}
        \hline
        & \multicolumn{5}{c|}{Simulation 1} & \multicolumn{5}{c}{Simulation 2} \\
        \cline{2-11}
        ~ & DInCIKFc & DInCIKF & SCI Lie~\cite{li2021joint} & InEKF~\cite{li2022closed} & & DInCIKFc & DInCIKF & SCI Lie~\cite{li2021joint} & InEKF~\cite{li2022closed} \\
        \hline
        Robot 1 & .509/.004 & .629/.004 & .766/.005 & 1.702/.020 & & .043/.003 & .043/.005 & .066/.005 & .272/.017 \\
        Robot 2 & .038/.001 & .039/.001 & .033/.002 & .922/.026 & & .111/.001 & .141/.003 & .233/.004 & .344/.006 \\
        Robot 3 & .081/.002 & .084/.002 & .175/.002 & .812/.024 & & .069/.008 & .147/.003 & .075/.002 & .365/.005 \\
        Robot 4 & .073/.001 & .073/.001 & .099/.002 & .991/.014 & & .095/.012 & .068/.007 & .099/.014 & .327/.012 \\
        Robot 5 & .167/.002 & .167/.002 & .210/.002 & .991/.013 & & .090/.008 & .066/.005 & .104/.082 & .345/.016 \\
        Average Robot& \textbf{.173}/\textbf{.002} & .198/.002 & .240/.003 & 1.038/.021 & & \textbf{.082}/.006 & .093/\textbf{.005} & .115/.008 & .331/.011 \\
        Average Object & \textbf{.427}/\textbf{.005} & .600/.008 & 1.101/.017 & 1.210/.132 & & \textbf{.087}/\textbf{.024} & .178/.033 & .110/.047 & .658/.087 \\
        \hline
    \end{tabular}
    \begin{tablenotes}
        \item[*] $\overline{\operatorname{RMSE}}_{pos}$ (m)/~$\overline{\operatorname{RMSE}}_{rot}$ (rad).
    \end{tablenotes}
    \label{tab:sim_result_combined}
    \vspace{-5mm}  
\end{table*}

We set up two scenarios. In the first scenario (Fig.\ref{fig:sim1}), robots 1 and 5 intermittently enter the perception blind spots and need to rely on communication with other robots for localization. Covariance matrices for process and observation noise are set to $0.01 I_6$.
In the second scenario (Fig.\ref{fig:sim2}), we use CoppeliaSim to simulate a $20 \times 10~m^2$ office with key objects like chairs, desks, and laptops. This scenario tests the broadcast of absolute pose information across a multi-robot communication network to align the local frames among robots, crucial for tasks like collaborative mapping. Robot 3 is initially biased by 1 meter along the x-axis, with the same noise settings as in the first scenario.

The estimation results for Scenario 1 are shown in Fig.\ref{fig:sim1}~(c) and Tab.\ref{tab:sim_result_combined}. It can be observed that Robot 1 and Robot 5 exhibit higher inaccuracy compared to the others, as they experienced blind periods. Our proposed algorithm outperforms both SCI Lie and InEKF, with InEKF performing the worst, illustrating that inconsistency leads to inaccuracy. Notably, DInCIKFc achieves the best performance in estimating object poses. The results for Scenario 2 are presented in Tab.\ref{tab:sim_result_combined} and Fig.\ref{fig:sim2}~(c). Here, our algorithm demonstrates a clear advantage in automatically compensating for initial bias. While Robot 3 suffers from a larger estimation error due to this bias, our method effectively corrects its trajectory. Additionally, leveraging common objects greatly enhances position estimation accuracy. The performance of DInCIKFc depends on the selected value of $N_o$, which was set to 6 in our experiments, as well as the number of common objects present in the environment.
\subsection{Real-world Dataset}
We use the YCB-Video dataset~\cite{xiang2017posecnn}, dividing it into three sequences with 30\% overlap to simulate three robots. We introduce noisy relative measurements with \( R_{ij} = 0.01^2 I_6 \). Following~\cite{9667208}, we use the process model~\eqref{eq:process1} and assume the constant velocity due to the low camera motion speed. Angular velocity is set to zero with noise, and the expected linear velocity is the average of the previous six estimations. We set the process noise to be \( 0.01^2 I_6 \). The communication network is a chain: Robot 1 communicates with Robot 2, Robot 2 with Robot 3, but Robot 1 and Robot 3 do not communicate directly.

To obtain measurements $^{i}T_{f,m}$ for robots $i$ and object $f$ we use the 6-DOF object poses in the current camera frame provided by FoundationPose~\cite{wen2024foundationpose}. Sample object measurements are illustrated in Fig.\ref{fig:ycbv}~(a)(c). In the real-world dataset experiments, we compare our method against two single-robot object SLAM algorithms: Gaussian Mixture Midway-Merge~(GM-midway)~\cite{jung2022gaussian} and the Right invariant EKF(RI-EKF)~\cite{9667208}. Additionally, we evaluate performance against the multi-robot method SCI on the Lie group.
\begin{table}[htbp] 
\centering\vspace{-1mm} 
\caption{Results of YCB-V dataset.}\vspace{-3mm} 
\begin{threeparttable}
\begin{tabular}{l|c|ccc}
\hline
    Sequence                &   ~    &07         &   22       &27    \\
\hline 
\multirow{2}*{DInCIKFc}  & Robot  &\textbf{.013}/\textbf{.010}  &\textbf{.012}/\textbf{.009}&\textbf{.014}/\textbf{.035 }\\             
                       ~&Object &\textbf{.020}/\textbf{.014}  &.018/\textbf{.011}& 
 \textbf{.015}/\textbf{.018 } \\
    \hline
\multirow{2}*{RI-EKF~\cite{9667208}}  & Robot  &.036/.055  &.028/.068&.017/.052 \\
                        ~&Object&.035/.036  &.023/.057&.019/.019\\
    \hline
\multirow{2}*{GM-midway~\cite{jung2022gaussian}}&Robot  &.035/.039  &.028/.032&.015/.039     \\                      
                        ~&Object&.034/.032  &.024/.046&.017/.019 \\
    \hline
\multirow{2}*{SCI Lie~\cite{li2021joint}} &Robot   &.027/.021  &.015/.012&.016/.036\\
                       ~&Object &.028/.034  &\textbf{.015}/.014&.017/.020 \\
\hline
\end{tabular}
\begin{tablenotes}
    \item $\overline{\operatorname{RMSE}}_{pos}$ (m)/~$\overline{\operatorname{RMSE}}_{rot}$ (rad).
    \item The results of RI-EKF and GM-midway are from~\cite{jung2022gaussian}.
\end{tablenotes}
\end{threeparttable}
\label{tab:ycbvresult}
\vspace{-2mm}  
\end{table}
As shown in Tab.\ref{tab:ycbvresult}, the multi-robot methods outperform the estimates generated by individual robots in general. This improvement arises from inter-robot communication, allowing more efficient information usage. Despite introducing relative measurement noise between robots, the information fusion can trade spatial resolution for temporal efficiency. These results highlight the importance of multi-robot systems in environmental exploration. Besides, the FoundationPose used for observation acquisition achieves higher accuracy benefitting for a better accuracy. The comparison between the estimation results of DInCIKF and SCI on Lie groups clearly shows that our method provides higher accuracy.
\begin{figure}[htbp]
    \centering
    \begin{minipage}[t]{0.44\linewidth}
        \centering
        \includegraphics[width=\linewidth]{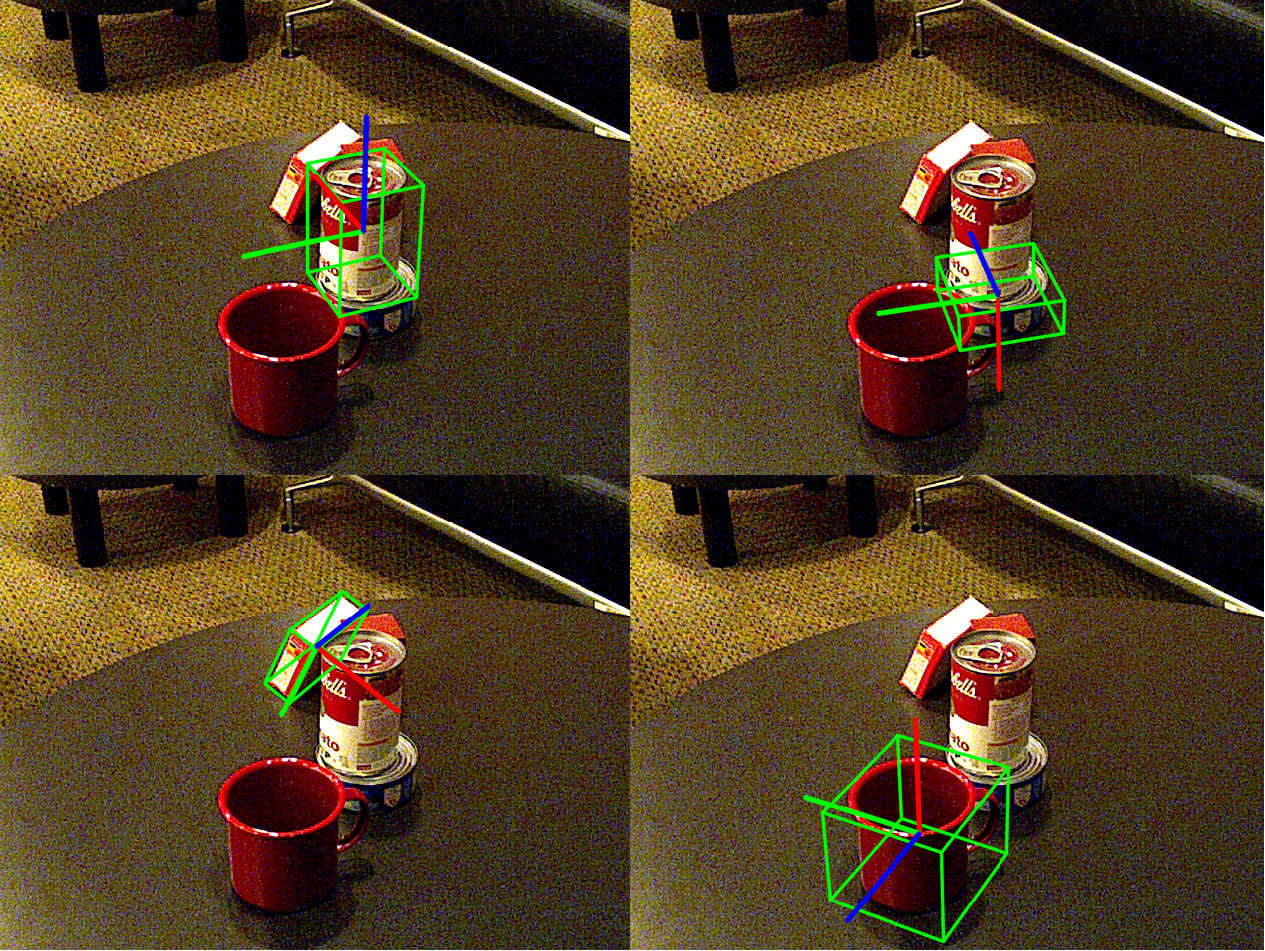}
        \text{(a)}
        \label{fig:seq22}
    \end{minipage}
    \hfill
    \begin{minipage}[t]{0.5\linewidth}
        \centering
        \includegraphics[width=\linewidth]{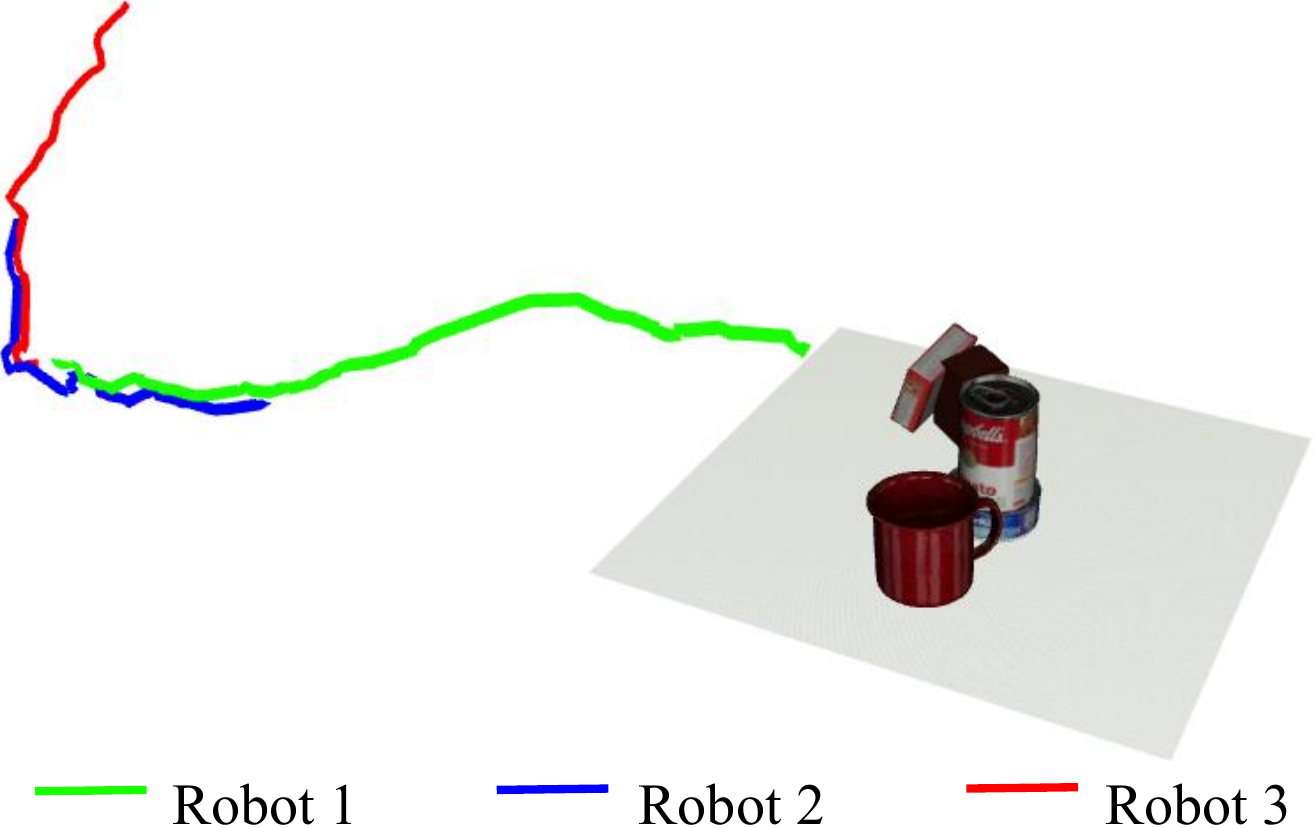}
        \text{(b)}
        \label{fig:seq22result}
    \end{minipage}
    \vspace{1em}
    \begin{minipage}[t]{0.44\linewidth}
        \centering
        \includegraphics[width=\linewidth]{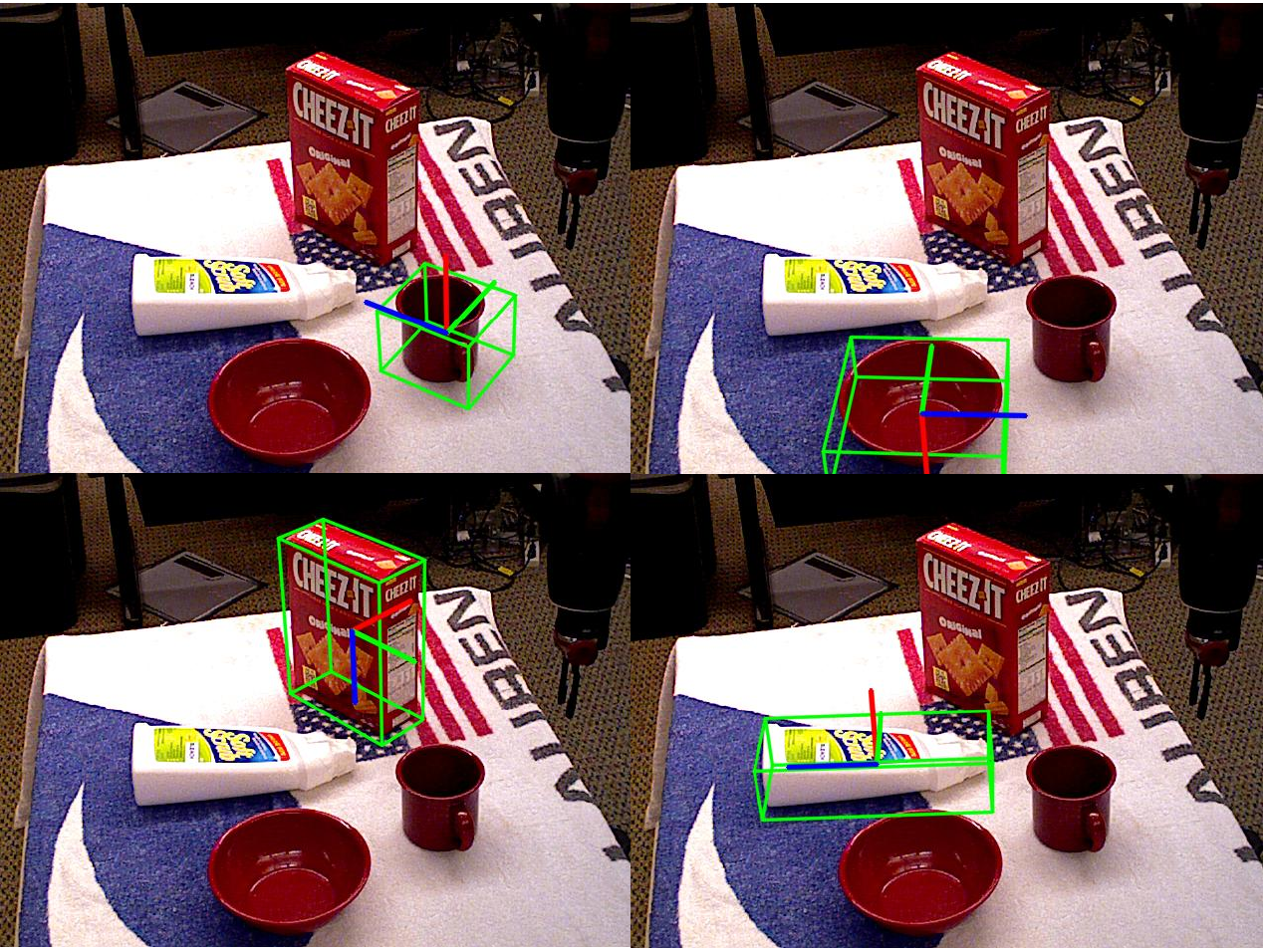}
        \text{(c)}
        \label{fig:seq7}
    \end{minipage}
    \hfill
    \begin{minipage}[t]{0.49\linewidth}
        \centering
        \includegraphics[width=\linewidth]{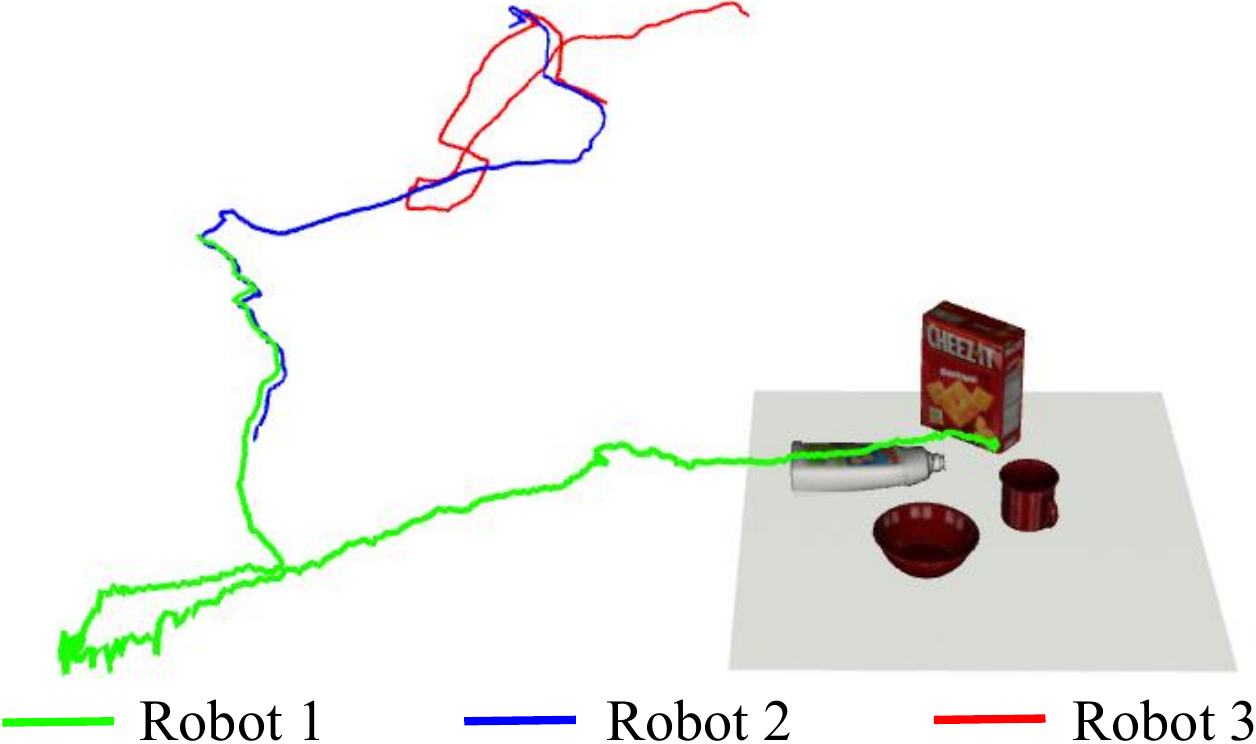}
        \text{(d)}
        \label{fig:seq7result}
    \end{minipage}
    \vspace{-4mm}  
    \caption{\textbf{YCB-V dataset Indication:} (a)(c) Sample images and FoundationPose outputs of sequences 22 and 7 respectively. (b)(d) Estimation results for sequences 22 and 7.}
    \label{fig:ycbv}
    \vspace{-3mm}  
\end{figure}
\section{Conclusion}
This paper presents DInCIKF, a fully distributed object-level pose SLAM algorithm designed to efficiently estimate both robot trajectories and object poses. By combining InEKF with CI, the approach mitigates overconfidence and conservatism in multi-robot systems. Modeling states on a Lie group and capturing uncertainties via Lie algebra improves estimation accuracy, linearity, and consistency. The use of object-level information further reduces communication demands, enhancing both bandwidth and frequency efficiency. Simulations and real-world experiments demonstrate the algorithm's robustness and effectiveness for real-time multi-robot applications. Future work will focus on refining map building, loop closure detection, and we are developing a much-needed object-level multi-robot SLAM dataset to address current data scarcity.

\newpage
\onecolumn
\section*{APPENDIX}
\subsection{Derivation of Jacobian matrix}
\label{app:Jacob}
Given that 
$$
y_{if}= \log( ^{i}T_{f,m}  ( \bar{T}_i^{-1}\bar{T}_{f}^{(i)})^{-1} )
$$

Then the Jacobian matrix of $y_{if}$ with respect to the logarithmic invariant error is:
\begin{align*}
I_6+y_{if}^\wedge & \approx \exp \left(y\right) &\\
& =T_i^{-1} \exp \left(n_{if}\right) T_f \hat{T}_f^{-1} \hat{T}_i \\
& =T_i^{-1} \exp \left(n_{if}\right) \exp \left(\delta_f\right) \hat{T}_i \\
& =T_i^{-1} \exp (\delta_i) \exp \left(-\delta_i\right) \exp \left(n_{i f}\right) \exp \left(\delta_f\right) \hat{T}_i \\
& =\hat{T}_i^{-1} \exp \left(-\delta_i+n_{i f}+\delta_f\right) \hat{T_i} \\
& \approx \hat{T}_i^{-1}\left(I-(\delta_i+n_{i f}+\delta_f)^\wedge\right) \hat{T_i} \\
& =I-\hat{T}_i^{-1}\left(n_{i f}-\delta_i+\delta_f\right)^\wedge \hat{T_i} \\
& =I-   \operatorname{Ad}_{\hat{T}_i^{-1}}         \left(n_{i f}-\delta_i+\delta_f\right)^\wedge.
\end{align*}
The linearization is then $
y_{if_s} = H_{if_s} \zeta_i + v',
$
where $v' = \operatorname{Ad}_{\hat{T}_i^{-1}}v $ and \begin{small}
$$H = \begin{bmatrix}
  (  \operatorname{Ad}_{\hat{T}_i^{-1}}G) & 0 & \cdots & 0 & -\operatorname{Ad}_{\hat{T}_i^{-1}}& 0 & \cdots & 0
\end{bmatrix}. $$\end{small}

\subsection{Proof of Lemma~\ref{lem:relation_fullrank}}
\label{App:pf_fullrank}
Given that 
$$\check{A} = \frac{1}{\sqrt{\alpha_i}}A=\begin{bmatrix}
    \check{F}&0\\
    0&\frac{1}{\sqrt{\alpha_i}} I
\end{bmatrix}
=\begin{bmatrix}
   \frac{1}{\sqrt{\alpha_i}}F&0\\
    0&\frac{1}{\sqrt{\alpha_i}} I\end{bmatrix}, \quad 
\check{{H}}_{j_1}(k)=\begin{bmatrix}
H_{i, n}(k)  \\
H_{i}(k)
\end{bmatrix}  = \begin{bmatrix}
\check{H}_j(k-1)A^{-1}J_1^\top J_1\\
J_i(k)\quad J_{if}(k)
\end{bmatrix}.
$$
For node $i \in \mathcal{V}_1$, 
\begin{align*}
\check{H}_i(k) &\triangleq \begin{bmatrix}
    H_{i,n}(k)\\
    H_{i}(k)
\end{bmatrix} = \begin{bmatrix}
    \check{H}_{j}(k-1) A^{-1} J_1^\top J_1\\
    H_i(k)
\end{bmatrix}=  \begin{bmatrix}
J_{j_1 o}(k-m+1) F^{-(m-1)} & 0\\
J_{j_1}(k-m+1) F^{-(m-1)} &0\\ 
J_{i}(k) & J_{if}(k)
\end{bmatrix}.
\end{align*}
Through row transformations and omit the positive scales, we obtain:
$
\operatorname{rank}(O(\check{A}, \check{{H}}_i)) = \operatorname{rank}(O'),
$
where $$O^{\prime}=\left[\begin{array}{cc}
O\left({F}, J_{j_1,o}\right) & 0 \\
J_i(k) & J_{if}(k)\\
J_i(k+1) A &J_{i f}(k+1) \\
J_i(k+2) A^2 & J_{i f}(k+2) \\
\vdots & \vdots
\end{array}\right]
$$
 The full rank of the first six or nine columns is achieved through $O(F, J_{j_1,o})$ has full rank. At the same time, the last six columns become full rank by nature of $H_{if}$.

\subsection{Proof of Lemma~\ref{Lem:upper}} \label{App:pf_upper}

By retaining only the parent nodes of robot \(i\) in the directed tree \(\mathcal{T}\), we obtain
\begin{align}
 &\hat{P}_i(k)^{-1}\notag\\
 &~=\gamma_{i,k} \breve{P}_i(k)^{-1}+\left(1-\gamma_{i, k}\right) J_1^{\top}\tilde{P}_i(k)^{-1} J_1 \notag\\
&~= \gamma_{i, k}\left(\bar{P}_i(k)^{-1}+H_i(k)^{\top} R_i^{-1} H_i(k)\right) \notag\\
&~~+\left(1-\gamma_{i, k}\right)J_1^{\top}\sum_{j \in \mathcal{N}_i} \beta_{j, k}\left(J_1\left(A \hat{P}_{j}(k-1) A^{\top}+Q_j\right) J_1^{\top}+R_{i j}\right)^{-1} J_1,\notag\\
&\geq \gamma_{i, k} \bar{P}_i(k)^{-1}+\left(1-\gamma_{i, k}\right) \Omega_j(k)+\gamma_{i,k} H_{i}(k)^{\top} R_{i}^{-1} H_{i}(k),~\label{eq:drop}
\end{align}
where $J_1\triangleq[G~0_{6\times6}]$ and robots share the same $A\triangleq \operatorname{diag}(F~0_{6 \times 6})$ for all estimating one object, and
$$\Omega_j(k)\triangleq \beta_{j, k}\left(J_1\left(A \hat{P}_{j}(k-1) A^{\top}+Q_j\right) J_1^{\top}+R_{i j}\right)^{-1} J_1.$$ Invoking Lemma 4 in~\cite{li2024covarianceintersectionbasedinvariantkalman}, if follows
\begin{equation}~\label{eq:Omega}
\Omega_j(k)\geq \eta_j J_1^{\top} J_1 A^{-\top}\left(\check{H}_{j}(k-1)^{\top} \check{R}_j^{-1} \check{H}_{j}(k-1)\right) A^{-1} J_1^{\top} J_1,
\end{equation}
where the calculation of $\eta_j$ see Appendix of~\cite{li2024covarianceintersectionbasedinvariantkalman}. 
Take \eqref{eq:Omega} into \eqref{eq:drop}, it follows that
\begin{align*}
 &\hat{P}_i(k)^{-1} \\
  &\geq \gamma_{i, k} \bar{\Pi}_i(k)^{-1}+\left(1-\gamma_{i, k}\right) \eta_j J_1^{\top} J_1 A^{-\top}\left(\check{H}_{j}(k-1)^{\top} \check{R}_j^{-1} \check{H}_{j}(k-1)\right) A^{-1} J_1^{\top} J_1 +\gamma_{i,k} H_{i}(k)^{\top} R_{i}^{-1} H_{i}(k)\\
  &\triangleq  \gamma_{i, k} \bar{\Pi}_i(k)^{-1}+\check{H}_i ^\top \check{R}_i^{-1}\check{H}_i(k)=\hat{\Pi}_i(k),
\end{align*}
where
$$
\check{H}_i(k) \triangleq\left[\begin{array}{c}
\check{H}_j(k-1) A^{-1} J_1^{\top} J_1 \\
H_i(k)
\end{array}\right] \quad \check{R}_i^{-1} \triangleq\left[\begin{array}{cc}
\left(1-r_i\right) \eta_j \check{R}_j^{-1} &0\\
0&\gamma_i R_i^{-1}
\end{array}\right].
$$
We repeat the above procedure to show $\hat{\Pi}_i(k) \geq \hat{P}_i(k)$
for $i\in \mathcal{V}_{3}$, and so on and so forth, which completes the proof.

\newpage

\bibliographystyle{unsrt}
\bibliography{ref}

\end{document}